\documentclass[runningheads]{llncs}
\usepackage[T1]{fontenc}
\usepackage{graphicx,verbatim}
\usepackage{algorithm}
\usepackage{algorithmic}
\usepackage{amsmath,amssymb,amsfonts}
\usepackage{algorithmic}
\usepackage{booktabs}       
\usepackage{nicefrac}       
\usepackage{microtype}      
\usepackage{xcolor}         
\usepackage{soul,color}
\usepackage{multirow}
\usepackage{rotating}
\usepackage{vcell}
\usepackage{amssymb}
\usepackage{hhline}
\usepackage[shortlabels]{enumitem}
\usepackage{listings}
\usepackage{float}
\usepackage{wrapfig}
\usepackage{pifont}
\usepackage[normalem]{ulem}
\usepackage{arydshln}
\usepackage{colortbl}
\usepackage{tcolorbox}
\usepackage{minted}
\usepackage{mathtools}
\usepackage{makecell}
\usepackage[rightcaption]{sidecap}
\usepackage{capt-of}
\usepackage[pagebackref,breaklinks,colorlinks]{hyperref}
\begin{document}
\title{Self-Auditing
Parameter-Efficient Fine-Tuning\\for Few-Shot
3D Medical Image Segmentation}

\author{Son Thai Ly \and Hien V. Nguyen}  
\authorrunning{Son Thai Ly \and Hien V. Nguyen}
\institute{University of Houston \\
    \email{stly@cougarnet.uh.edu $\quad$ hvnguy35@central.uh.edu}}
\maketitle            
\begin{abstract}
Adapting foundation models to new clinical sites remains challenging in practice. Domain shift and scarce annotations must be handled by experts, yet many clinical groups do not have ready access to skilled AI engineers to tune adapter designs and training recipes. As a result, adaptation cycles can stretch from weeks to months, particularly in few-shot settings. Existing PEFT methods either require manual adapter configuration or automated searches that are computationally infeasible in few-shot 3D settings. We propose \textbf{SEA-PEFT} (\textit{self-auditing parameter-efficient fine-tuning}) to automate this process. SEA-PEFT treats adapter configuration as an online allocation problem solved during fine-tuning rather than through manual, fixed-topology choices. SEA-PEFT uses a \textit{search--audit--allocate} loop that trains active adapters, estimates each adapter's Dice utility by momentarily toggling it off, and then reselects the active set under a parameter budget using a greedy knapsack allocator. Exponential Moving Average and Interquartile Range smoothing, together with a Finite-State Ranking controller, stabilize the loop and improve reliability in high-noise few-shot regimes. On TotalSegmentator and FLARE'22, SEA-PEFT improves mean Dice by 2.4--2.8 points over the strongest fixed-topology PEFT baselines across 1/5/10-shot settings while training $<$1\% of parameters. For reproducibility purposes, we made our code publicly available at \url{https://github.com/tsly123/SEA_PEFT}

\end{abstract}

\section{Introduction}
Recent foundation models (FMs) for 3D medical imaging, including Swin-UNETR \cite{hatamizadeh2021swin}, SuPreM \cite{li2024well}, and FSEFT \cite{silva2023towards}, have produced strong results in organ segmentation across CT and MRI. Adapting these models to a new clinical site is nevertheless non-trivial. Hospitals vary in scanner protocols, reconstruction kernels, and patient demographics, and this distribution shift requires site-specific fine-tuning. At the same time, volumetric mask annotation is labour-intensive, most sites lack dedicated AI staff, and 3D patch training imposes significant memory requirements. As a result, fine-tuning a medical FM in a new hospital typically requires an AI engineer, labelled data, and careful configuration choices that are unavailable in most clinical settings.

Parameter-Efficient Fine-Tuning (PEFT) \cite{lora,adaptformer,bias1,adapter1,vpt} reduces the number of updated parameters and is therefore a natural candidate for low-resource clinical adaptation. The problem is that PEFT still requires three configuration decisions: which adapter type to use, where in the transformer to insert it, and what rank or bottleneck dimension to assign. These choices interact and depend on the target organ, dataset size, and acquisition protocol, so there is no universally good default. Offline search methods such as NOAH \cite{zhang2024neural}, AutoPEFT \cite{zhou2024autopeft}, and Fairtune \cite{dutt2023fairtune} can find good configurations but require multiple end-to-end fine-tuning runs, which is impractical for few-shots with high-resolution 3D volumes.

We propose to solve the PEFT configuration problem online, during fine-tuning, using direct task-level feedback rather than weight-space proxies. The core idea is simple: to assess whether an adapter is useful, temporarily remove it and measure the change in performance on a held-out validation set. This on/off perturbation gives an unbiased, model-agnostic estimate of each adapter's marginal performance contribution. Repeating this measurement across training, smoothing the estimates, and greedily selecting adapters under a parameter budget yields a configuration that improves as more audits accumulate, without requiring any offline sweep.

We proposed framework \textbf{SEA-PEFT} (\underline{SE}lf-\underline{A}uditing PEFT). The search--audit--allocate loop trains active adapters for $K$ steps (\textbf{Search}), estimates per-adapter utility via on/off toggling (\textbf{Audit}), and selects the active set under a global budget (\textbf{Allocate}). An Exponential Moving Average (EMA) $+$ Interquartile Range (IQR) utility tracker stabilizes noisy audit estimates, and a Finite-State Machine (FSM) prevents premature adapter switching by requiring several consistent votes before committing a configuration change. After the loop converges, a guard-free knapsack re-solve selects the final configuration, which is then fine-tuned from scratch so that the reported model carries no exploratory weights from the search phase. \textbf{Contributions}: (\textit{i}) We propose SEA-PEFT, an online PEFT configuration framework for few-shot 3D medical segmentation that selects adapter type, placement, and rank via task-level perturbation auditing during fine-tuning, with no offline search and no manual configuration; (\textit{ii}) The novel couples with EMA$+$IQR utility tracker and FSM stabilizer as practical components for reliable adapter selection under few-shot noise; (\textit{iii}) We provide both theoretical and empirical evidence on TotalSegmentator and FLARE datasets to demonstrate the effectiveness of SEA-PEFT.

\vspace{-10pt}
\section{Related Works}\label{sec:related}
\noindent\textbf{PEFT for medical imaging.} Adapter-based methods \cite{adapter1,adaptformer}, LoRA and variant \cite{lora,vera,fourierft}, BitFit \cite{bias1}, and prompt tuning \cite{li2021prefix,liu2022p,vpt} all adapt pretrained models by updating a small parameter subset while keeping the backbone frozen. For 3D medical FMs, FSEFT \cite{silva2023towards,baklouti2025regularized}, and FreqFit \cite{ly2025frequency} apply fixed single-adapter designs to few-shot organ segmentation. In all of these works, adapter type, insertion site, and rank are chosen before training, either by hand or via grid search on a held-out set. SEA-PEFT eliminates this step by selecting and adjusting the configuration during training.

\noindent\textbf{Automated PEFT configuration.} NOAH \cite{zhang2024neural} unifies adapters, prompts, and LoRA in a supernet and runs a one-shot NAS to find a good prompt configuration. AutoPEFT \cite{zhou2024autopeft} and Fairtune \cite{dutt2023fairtune} use multi-objective Bayesian optimization to search over adapter type, placement, and rank combinations. Both approaches require instantiating and fully fine-tuning each candidate configuration, which is not feasible for high-resolution 3D segmentation at 1--10 shots. 
AdaLoRA \cite{zhang2023adalora} reallocates LoRA rank during training based on singular value magnitudes, avoiding full re-training, but it optimizes an internal weight-space criterion rather than directly measuring Dice utility per adapter. SEA-PEFT instead evaluates each adapter's contribution to segmentation accuracy through on/off perturbation, which is both model-agnostic and directly aligned with the evaluation metric.
\vspace{-10pt}
\section{Self-Aufit PEFT}\label{sec:method}
\subsection{Problem Setup and Audit Space}
Let $f_\theta$ denote a frozen FM backbone. The \emph{audit space} $\mathcal{A}=\{a_i\}_{i=1}^{N}$ is a library of lightweight adapter candidates that can be activated, deactivated, or resized during training. Each unit $a_i$ has trainable weights $w_i$, binary gate $\alpha_i\in\{0,1\}$, and parameter cost $c_i$.

\begin{figure}[t]
\vspace{-3mm}
\noindent
\begin{minipage}[t]{0.48\textwidth}
  \vspace{0pt} 
  \centering
  \includegraphics[width=\linewidth]{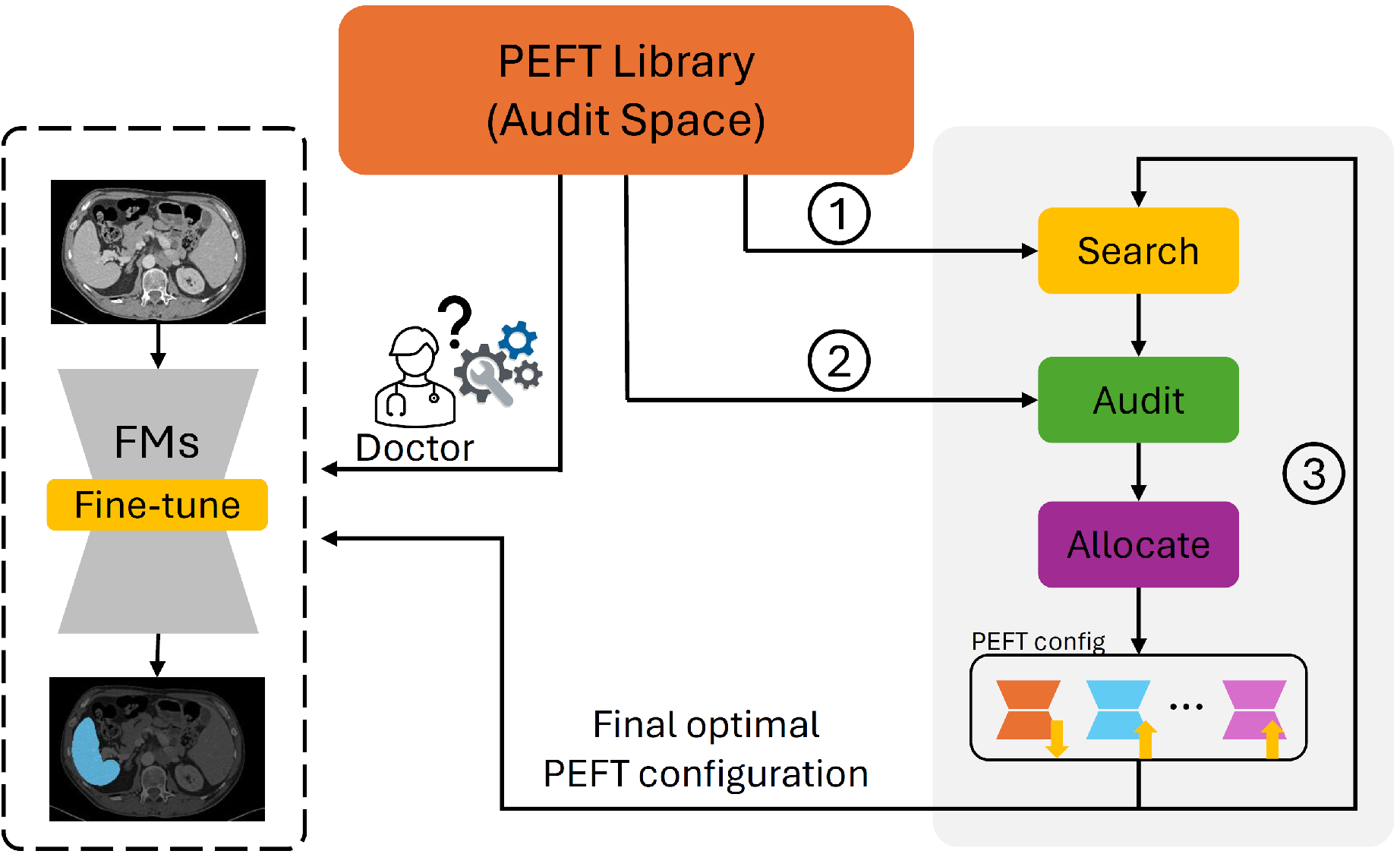} 
\end{minipage}\hfill
\begin{minipage}[t]{0.48\textwidth}
  \vspace{0pt} 
\captionof{figure}{SEA-PEFT iterates over a PEFT library with three steps: (1) Search selects a small adapter set under a parameter budget; (2) Audit runs lightweight on/off tests to estimate Dice-per-cost utility; (3) Allocate updates the active configuration via budget-aware rank adjustments. This enables rapid adaptation by clinical users without AI engineers.}  \label{fig:overview}
\end{minipage}
\vspace{-5mm}
\end{figure}

\noindent\textit{Audit space.} Following standard practice for neural architecture search (NAS) studies \cite{ru2020neural,xie2019exploring,yang2019evaluation}, we adopt three adapter topologies: Serial (\textbf{SA}), Parallel (\textbf{PA}), and their composite \textbf{SAPA}. For layer output $F(x)$ with projections $W_{down}\in\mathbb{R}^{D\times d}$ and an up-projection $W_{up}\in\mathbb{R}^{d\times D}$ ($d\!\ll\!D$): \vspace{-2pt} \begin{align}
    \textit{SA:}\quad y = \:\:&\text{ReLU}(F(x)\cdot W_{down})\cdot W_{up}\\
    \textit{PA:}\quad y = \:\:&\text{ReLU}(x\cdot W_{down})\cdot W_{up}\\
    \textit{SAPA:}\quad y = \:\:&\text{ReLU}(F(x)\cdot W_{down})\cdot W_{up} + \text{ReLU}(x\cdot W_{down})\cdot W_{up}
\end{align}

\noindent\textbf{Objective.} Given sets $\mathcal{D}_{train}$, $\mathcal{D}_{val}$, and budget $P_{\max}$: \vspace{-2pt} \vspace{-2pt}
\begin{equation}
\max_{\alpha,w}~\mathbb{E}_{(x,y)\sim\mathcal{D}_{val}}
[\text{Dice}(f_{\theta,\alpha,w}(x),y)]
\quad\text{s.t.}\sum_i\alpha_i c_i \leq P_{\max}
\label{eq:obj}
\vspace{-2mm}
\end{equation}

This mirrors a knapsack problem where each adapter $a_i$ has cost $c_i$ and value equal to its Dice improvement. The  utilities are unknown and must be estimated from validation feedback, SEA-PEFT solves an online variant of Eq. \eqref{eq:obj}.

\subsection{Search--Audit--Allocate Cycle}

\noindent\textbf{Search.} SEA-PEFT trains currently active adapters ($\alpha_i\!=\!1$) on $\mathcal{D}_{train}$ while keeping the backbone frozen. After $K$ steps, the framework transits to audit.

\noindent\textbf{Audit.} A mini-batch $\mathcal{M}^{(t)}\subset\mathcal{A}$ of $M\!\ll\!N$ adapters is sampled using a \textit{active/inactive}: empirically 30\% from currently active adapters (exploitation) and 70\% from inactive ones (exploration), plus an $\varepsilon$-exploration term up-weighting rarely probed adapters. For each $i\in\mathcal{M}^{(t)}$:
\vspace{-2pt}
\begin{align}
u_i &= \left[\text{Dice}(f_{\theta,\alpha},\mathcal{D}_{val}) -
            \text{Dice}(f_{\theta,\alpha\setminus i},\mathcal{D}_{val})\right]/c_i\label{eq:util}\\        
r^{(t)}_i &= \text{median}_w(\tilde{u}^t_i)-\lambda_S\cdot\text{IQR}_w(\tilde{u}^t_i) \quad\text{with} \quad\tilde{u}^t_i &= (1-\beta)u^t_i + \beta\tilde{u}^{(t-1)}_i\label{eq:iqr} 
\vspace{-2mm}
\end{align}
\noindent\textit{Utility smoothing.} Raw utilities are noisy because only $M$ adapters are probed per cycle. SEA-PEFT applies a two-stage filter as in Eq. \eqref{eq:iqr}.

Each adapter maintains a history deque of 3--5 audits. High $r^{(t)}_i$ signals stable Dice improvement; large IQR suppresses unstable adapters.

\begin{proposition}[Convergence of the EMA+IQR utility estimator]\label{prop1}
Let $u_i^{(t)} = \mu_i^{(t)} + \xi_i^{(t)}$ where $\mathbb{E}[\xi_i^{(t)}\!\mid\!\mathcal{F}_{t-1}]=0$, $\mathrm{Var}(\xi_i^{(t)})\leq\sigma^2$, and $\mu_i^{(t)}$ drifts by at most $\delta$ between consecutive audits. After $t$ audits with Exponential Moving Average (EMA) decay $\beta\in(0,1)$:
\begin{equation}
\Bigl|\mathbb{E}[\tilde{u}_i^{(t)}] - \mu_i^{(t)}\Bigr| \;\leq\; \frac{\delta\beta}{1-\beta}
+ \beta^t\Bigl|\mathbb{E}[\tilde{u}_i^{(0)}] - \mu_i^{(t)}\Bigr|,
\qquad
\mathrm{Var}(\tilde{u}_i^{(t)})
\;\leq\; \frac{(1-\beta)\sigma^2}{1+\beta}
\vspace{-2mm}
\end{equation}
Furthermore, over a sliding history window of audits, the sample Interquartile Range (IQR) $\mathrm{IQR}_H(\tilde{u}_i^{(t)})$ concentrates around the population IQR of the noise distribution with fluctuations of order $O(\sigma_{\mathrm{eff}}/\sqrt{H})$, where $\sigma_{\mathrm{eff}}=\sigma\sqrt{(1-\beta)/(1+\beta)}$ is the effective noise scale of the smoothed estimates. The robust score $r_i^{(t)} = \mathrm{median}_H(\tilde{u}_i^{(t)}) - \lambda_S\cdot\mathrm{IQR}_H(\tilde{u}_i^{(t)})$, with IQR shrinkage weight $\lambda_S>0$, suppresses adapters whose utility is highly variable across audits.
\end{proposition}
\begin{proof}[Proof sketch]
The bias bound follows by expanding the EMA recursion, applying $|\mu_i^{(t-k)}-\mu_i^{(t)}|\leq k\delta$, and bounding $(1-\beta)\sum_{k=1}^{\infty}k\beta^k = \delta\beta/(1-\beta)$. The variance bound uses independence of $\{\xi_i^{(t)}\}$ and $\sum_{k\geq 0}\beta^{2k}=(1-\beta^2)^{-1}$. IQR concentration follows standard order-statistics theory under bounded variance.
\end{proof}

\begin{lemma}[Audit coverage]\label{lemma1}
Under any sampling strategy with exploration rate $\varepsilon>0$, audit set $\mathcal{M}^{(t)}$ (size $M$) at cycle $t$, and $|\mathcal{A}|=N$, every adapter $a_i\in\mathcal{A}$
\begin{equation}
\mathbb{E}\!\left[\,\sum_{t=1}^{T} \mathbf{1}[a_i\in\mathcal{M}^{(t)}]\right] \;\geq\; \rho_c\, T,
\qquad \rho_c = \frac{\varepsilon M}{N} > 0
\vspace{-2mm}
\end{equation}
satisfying the positive coverage condition required for sublinear regret in \cite{streeter2008online}. That is, every adapter is audited $\Omega(T)$ times regardless of its initial activity state.
\end{lemma}
\begin{proof}[Proof sketch]
At each audit cycle the $\varepsilon$-exploration term assigns a minimum selection probability of $\varepsilon M/N$ to every adapter independently of its activity gate $\alpha_i$. Summing over $T$ cycles and taking expectations gives the bound directly. Without $\varepsilon$-exploration, $\rho_c=0$ and inactive adapters at initialization are never audited, making the final knapsack re-solve arbitrarily suboptimal.
\end{proof}
\noindent\textbf{Allocate.} Adapters are greedily activated in descending density order $d_i=r_i/c_i$ until $\sum_i\alpha_i c_i\leq P_{\max}$. Replacements are accepted iff improvement exceeds $\mu_\text{eff}$.

\noindent\textbf{Finite-State Machine Stabilizer.}
Na\"ive greedy allocation allows adapters to flip on/off at nearly every cycle, causing \emph{configuration chatter} that slows convergence and destabilizes EMA utility estimates. The FSM requires $\tau_\text{act}$ consecutive consistent ``votes'' before committing an activity change (and $\tau_\text{rank}$ for rank):
\begin{equation}
\alpha^{(t+1)}_i \leftarrow \hat{\alpha}^{(t+1)}_i~\text{ only if }~c_i^\text{act} \geq \tau_\text{act}
\label{eq:fsm}
\end{equation}
with $\tau_\text{act},\tau_\text{rank}\approx 2$--$3$. FSM reduces reconfigurations to $T_c\approx 20$--$30$ per run, suppressing noise-driven instability while preserving online adaptability

\begin{proposition}[FSM reduces structural error]\label{prop2}
Let $T_c$ denote the number of audit cycles in which the active set changes, and let $\hat{\alpha}_i^{(t+1)}\in\{0,1\}$ be the activity proposed by the greedy allocator at cycle $t$. The FSM maintains a per-adapter vote counter $c_i^{\mathrm{act}}$, incremented when $\hat{\alpha}_i^{(t+1)}\neq\alpha_i^{(t)}$ and reset otherwise, and commits a change only when $c_i^{\mathrm{act}}\geq\tau_{\mathrm{act}}$ consecutive consistent proposals are observed. This gives:
\begin{equation}
T_c \;\leq\; \left\lfloor \frac{T}{\tau_{\mathrm{act}}} \right\rfloor
\end{equation}
In the regret decomposition adapting \cite{streeter2008online}, combined with the structural error term $O(\varepsilon P_{\max} T_c)$ in the regret decomposition, the FSM reduces structural error by a factor of $\tau_{\mathrm{act}}$, tightening the knapsack approximation guarantee of \cite{sviridenko2004note}.
\end{proposition}
\begin{proof}[Proof sketch]
A configuration change at cycle $t$ requires $\tau_{\mathrm{act}}$ consecutive cycles with a consistent new proposal from the allocator. Hence at least $\tau_{\mathrm{act}}$ cycles must elapse between any two changes, giving $T_c\leq\lfloor T/\tau_{\mathrm{act}}\rfloor$ by direct counting. Hence, reducing the structural error term $O(\varepsilon P_{\max}T_c)$ by a factor of $\tau_{\mathrm{act}}$ and tightening the final approximation guarantee.
\end{proof}

\textbf{Final Re-solve and Re-fine-tuning.} After $T$ steps, SEA-PEFT performs one \emph{guard-free} knapsack re-solve on stable scores $\{r^{(T)}_i\}$ using density-greedy plus a single-swap local improvement pass, yielding $\mathcal{S}^\star_\text{final}$. Adapter weights and optimizer states are then re-initialized, and a standard fine-tuning run with only the selected adapters active is performed. This two-phase design (online search $+$ focused re-finetune) ensures the final model is not burdened by exploratory weights from the audit phase.

\vspace{-10pt}
\section{Experiments} \label{sec:experiments}
\textbf{Audit space}. 
Following the empirical experience in prior studies \cite{lora,adaptformer,basu2024strong,ly2025frequency,silva2023towards,baklouti2025regularized}, we fine-tune the following: For self-attention layer: we attached LoRA \cite{lora} with rank $\in \{2, 4, 8, 16\}$ with SA, PA, and SAPA. Feed-forward layer: we attached LoRA \cite{lora} with rank $\in \{2, 4, 8, 16\}$ and AdaptFormer \cite{adaptformer} with bottleneck $\in \{4, 8, 16, 32\}$, both with SA, PA, and SAPA. Normalization layer: we attached Affine-LN \cite{basu2024strong}.

\textbf{Datasets and backbones}. We evaluate on nine abdominal organs (spleen, left kidney, gallbladder, esophagus, liver, pancreas, stomach, duodenum, aorta) under 1-, 5-, and 10-shot regimes. Binary segmentation is performed on TotalSegmentator \cite{wasserthal2023totalsegmentator} and multi-class on FLARE \cite{ma2024unleashing}, using pre-trained FSEFT \cite{silva2023towards} and SuPreM \cite{li2024well} on frozen Swin-UNETR \cite{hatamizadeh2021swin} backbones.

Note that, NOAH \cite{zhang2024neural} and AutoPEFT \cite{zhou2024autopeft} require full tuning per candidate configuration. Over the audit space at 1--10 shots, this would require hundreds of 3D training runs, SEA-PEFT estimates adapter utility via on/off perturbation within a single run, making direct comparison infeasible and arguably unfair.

\textbf{Implementation}. After the search--audit--allocate loop converges, following \cite{silva2023towards,baklouti2025regularized,ly2025frequency}, the selected configuration is fine-tuned image size of $96 \times 96 \times 96$ patches with batch size 1, AdamW, lr = $1\times e^{-3}$, 200 epochs with early stopping. Augmentation includes intensity shifts and 90 degree rotations. All experiments run on a single V100 32\,GB GPU. Results are averaged across 3 random seeds on hold-out test set. The metric presented is Dice.

\vspace{-10pt}
\subsection{Results and Discussion}
\textbf{Main results.} Tables \ref{tab:totalseg} reports few-shot adaptation results on TotalSegmentator \cite{wasserthal2023totalsegmentator} and FLARE’22 \cite{ma2024unleashing} using the final configurations selected by our \textit{Search--Audit--Allocate} loop. Across all shot regimes and both datasets, SEA-PEFT consistently achieves the highest mean Dice, outperforming fixed-topology PEFT baselines without any manual tuning of adapter type, placement, or rank. On TotalSegmentator, it attains 78.33/79.90/80.29 Dice in the 1-, 5-, and 10-shot settings, a clear improvement over the strongest baseline. On FLARE’22, SEA-PEFT achieves 75.77 (1-shot) and 76.33 (5-shot), continuing the same upward trend as supervision increases. These gains demonstrate that online self-auditing reliably generalizes across distinct clinical datasets.

\begin{table*}[!t]
\centering
\setlength{\extrarowheight}{0pt}
\addtolength{\extrarowheight}{\aboverulesep}
\addtolength{\extrarowheight}{\belowrulesep}
\setlength{\aboverulesep}{0pt}
\setlength{\belowrulesep}{0pt}
\caption{Few-shot efficient adaptation on TotalSegmentator \cite{wasserthal2023totalsegmentator} and FLARE \cite{ma2024unleashing} with pre-trained FSEFT \cite{silva2023towards}. \textbf{Bold} and \uline{underscore} indicates the best and second best result.}\label{tab:totalseg}
\resizebox{\textwidth}{!}{
\small
\renewcommand{\arraystretch}{0.6}
\begin{tabular}{cclcccccccccc} 
\toprule
\multicolumn{2}{c}{Setting} & \multicolumn{1}{c}{Method} & Spl & lKid & Gall & Eso & Liv & Pan & Sto & Duo & Aor & \textbf{Avg.} \\ 
\hline
\multirow{18}{*}{\rotatebox[origin=c]{90}{TotalSeg (binary)}} & \multirow{6}{*}{1-shot} & Full Fine-tuning & 91.31 & 88.00 & 76.58 & \uline{49.11} & \textbf{98.80} & 77.31 & 68.54 & \textbf{58.24} & \textbf{81.17} & \uline{75.90} \\ 
\cline{3-13}
 &  & BitFit & 88.42 & 87.89 & 75.49 & 48.23 & 92.00 & 74.92 & 70.27 & 49.95 & 72.03 & 73.24 \\
 &  & LoRA & 90.38 & 88.68 & 77.63 & 44.54 & 92.48 & \uline{78.13} & \uline{72.19} & 54.75 & 70.28 & 74.34 \\
 &  & AdaptFormer & 89.85 & 88.75 & \uline{81.22} & 47.64 & 89.95 & 68.97 & 70.63 & \uline{57.45} & 76.47 & 74.55 \\
 &  & Affine-LN & \uline{91.62} & \uline{88.87} & 79.34 & 45.57 & 90.17 & 75.05 & 67.43 & 49.76 & 71.78 & 73.29 \\
 &  & {\cellcolor[rgb]{0.875,0.875,0.875}}SEA-PEFT & \multicolumn{1}{l}{{\cellcolor[rgb]{0.875,0.875,0.875}}\textbf{93.32}} & \multicolumn{1}{l}{{\cellcolor[rgb]{0.875,0.875,0.875}}\textbf{89.94}} & \multicolumn{1}{l}{{\cellcolor[rgb]{0.875,0.875,0.875}}\textbf{85.04}} & \multicolumn{1}{l}{{\cellcolor[rgb]{0.875,0.875,0.875}}\textbf{50.13}} & \multicolumn{1}{l}{{\cellcolor[rgb]{0.875,0.875,0.875}}\uline{93.42}} & \multicolumn{1}{l}{{\cellcolor[rgb]{0.875,0.875,0.875}}\textbf{79.68}} & \multicolumn{1}{l}{{\cellcolor[rgb]{0.875,0.875,0.875}}\textbf{75.75}} & \multicolumn{1}{l}{{\cellcolor[rgb]{0.875,0.875,0.875}}56.95} & \multicolumn{1}{l}{{\cellcolor[rgb]{0.875,0.875,0.875}}\uline{80.71}} & \multicolumn{1}{l}{{\cellcolor[rgb]{0.875,0.875,0.875}}\textbf{78.33}} \\ 
\hhline{~============}
 & \multirow{6}{*}{5-shot} & Full Fine-tuning & 90.47 & 89.02 & 71.88 & \textbf{54.33} & \textbf{93.75} & 78.20 & 59.26 & \uline{66.27} & \textbf{91.20} & \uline{77.15} \\ 
\cline{3-13}
 &  & BitFit & 91.46 & 88.26 & 71.37 & 50.50 & 92.96 & 57.94 & 71.22 & 56.91 & 73.65 & 72.70 \\
 &  & LoRA & \uline{92.18} & 89.08 & 73.69 & 48.98 & 93.18 & \uline{78.27} & 72.73 & 60.18 & 78.43 & 76.30 \\
 &  & AdaptFormer & 90.32 & 89.49 & 74.93 & 45.54 & 93.16 & 71.36 & 74.01 & 64.62 & 83.22 & 76.29 \\
 &  & Affine-LN & 91.75 & \uline{90.19} & \uline{76.86} & 47.51 & 93.30 & 77.25 & \uline{74.15} & 63.56 & 74.66 & 76.58 \\
 &  & {\cellcolor[rgb]{0.875,0.875,0.875}}SEA-PEFT & \multicolumn{1}{l}{{\cellcolor[rgb]{0.875,0.875,0.875}}\textbf{93.72}} & \multicolumn{1}{l}{{\cellcolor[rgb]{0.875,0.875,0.875}}\textbf{90.68}} & \multicolumn{1}{l}{{\cellcolor[rgb]{0.875,0.875,0.875}}\textbf{81.21}} & \multicolumn{1}{l}{{\cellcolor[rgb]{0.875,0.875,0.875}}\uline{52.25}} & \multicolumn{1}{l}{{\cellcolor[rgb]{0.875,0.875,0.875}}\uline{93.42}} & \multicolumn{1}{l}{{\cellcolor[rgb]{0.875,0.875,0.875}}\textbf{80.23}} & \multicolumn{1}{l}{{\cellcolor[rgb]{0.875,0.875,0.875}}\textbf{76.56}} & \multicolumn{1}{l}{{\cellcolor[rgb]{0.875,0.875,0.875}}\textbf{66.30}} & \multicolumn{1}{l}{{\cellcolor[rgb]{0.875,0.875,0.875}}\uline{84.76}} & \multicolumn{1}{l}{{\cellcolor[rgb]{0.875,0.875,0.875}}\textbf{79.90}} \\ 
\hhline{~============}
 & \multirow{6}{*}{10-shot} & Full Fine-tuning & \uline{92.97} & 80.85 & 72.96 & \textbf{56.56} & 92.36 & 78.67 & 61.27 & 66.56 & \textbf{90.15} & 76.93 \\ 
\cline{3-13}
 &  & BitFit & 92.51 & 89.41 & 76.48 & 50.57 & 92.54 & 79.40 & 67.09 & 65.32 & 76.89 & 76.69 \\
 &  & LoRA & 91.91 & \uline{90.10} & \uline{82.66} & 49.36 & \uline{93.45} & \uline{80.86} & 67.97 & 61.36 & 82.47 & \uline{77.79} \\
 &  & AdaptFormer & 90.36 & 89.91 & 73.67 & 49.19 & 89.50 & 71.13 & 66.85 & 64.95 & 80.48 & 75.12 \\
 &  & Affine-LN & 92.20 & 86.02 & 79.58 & 50.27 & 89.98 & 77.64 & \uline{69.15} & \uline{67.64} & 83.53 & 77.33 \\
 &  & {\cellcolor[rgb]{0.875,0.875,0.875}}SEA-PEFT & \multicolumn{1}{l}{{\cellcolor[rgb]{0.875,0.875,0.875}}\textbf{93.35}} & \multicolumn{1}{l}{{\cellcolor[rgb]{0.875,0.875,0.875}}\textbf{90.57}} & \multicolumn{1}{l}{{\cellcolor[rgb]{0.875,0.875,0.875}}\textbf{83.96}} & \multicolumn{1}{l}{{\cellcolor[rgb]{0.875,0.875,0.875}}\uline{51.70}} & \multicolumn{1}{l}{{\cellcolor[rgb]{0.875,0.875,0.875}}\textbf{94.23}} & \multicolumn{1}{l}{{\cellcolor[rgb]{0.875,0.875,0.875}}\textbf{81.74}} & \multicolumn{1}{l}{{\cellcolor[rgb]{0.875,0.875,0.875}}\textbf{72.55}} & \multicolumn{1}{l}{{\cellcolor[rgb]{0.875,0.875,0.875}}\textbf{67.83}} & \multicolumn{1}{l}{{\cellcolor[rgb]{0.875,0.875,0.875}}\uline{86.72}} & \multicolumn{1}{l}{{\cellcolor[rgb]{0.875,0.875,0.875}}\textbf{80.29}} \\ 
\hline
\multirow{12}{*}{\rotatebox[origin=c]{90}{FLARE (multi-organ)}} & \multirow{4}{*}{1-shot} & LoRA & \uline{87.81} & 76.22 & 55.09 & \textbf{74.33} & \textbf{95.05} & \uline{82.60} & \uline{74.72} & 36.03 & 90.31 & 74.68 \\
 &  & Adaptformer & 83.69 & \textbf{76.77} & 54.49 & 72.13 & 94.24 & 80.01 & 73.85 & \uline{43.86} & \uline{90.79} & 74.43 \\
 &  & Affine-LN & 85.46 & \uline{76.55} & \textbf{55.91} & 72.23 & \uline{94.85} & 82.20 & 73.14 & \textbf{44.66} & \textbf{91.57} & \uline{75.17} \\
 &  & {\cellcolor[rgb]{0.875,0.875,0.875}}SEA-PEFT & {\cellcolor[rgb]{0.875,0.875,0.875}}\textbf{88.33} & {\cellcolor[rgb]{0.875,0.875,0.875}}74.69 & {\cellcolor[rgb]{0.875,0.875,0.875}}\uline{55.57} & {\cellcolor[rgb]{0.875,0.875,0.875}}\uline{74.15} & {\cellcolor[rgb]{0.875,0.875,0.875}}94.61 & {\cellcolor[rgb]{0.875,0.875,0.875}}\textbf{83.34} & {\cellcolor[rgb]{0.875,0.875,0.875}}\textbf{79.87} & {\cellcolor[rgb]{0.875,0.875,0.875}}40.70 & {\cellcolor[rgb]{0.875,0.875,0.875}}90.70 & {\cellcolor[rgb]{0.875,0.875,0.875}}\textbf{75.77} \\ 
\hhline{~============}
 & \multirow{4}{*}{5-shot} & LoRA & \uline{86.56} & \uline{76.15} & 54.24 & 74.46 & \uline{94.93} & \uline{83.01} & 75.17 & 43.66 & \textbf{92.01} & 75.58 \\
 &  & Adaptformer & 81.50 & 75.36 & \textbf{55.86} & 74.53 & 94.66 & 82.87 & \textbf{76.15} & \uline{46.89} & 91.43 & 75.47 \\
 &  & Affine-LN & 85.73 & 76.00 & 53.71 & \uline{75.20} & 94.62 & 82.04 & 74.69 & \textbf{47.24} & 91.38 & \uline{75.62} \\
 &  & {\cellcolor[rgb]{0.875,0.875,0.875}}SEA-PEFT & {\cellcolor[rgb]{0.875,0.875,0.875}}\textbf{87.15} & {\cellcolor[rgb]{0.875,0.875,0.875}}\textbf{77.97} & {\cellcolor[rgb]{0.875,0.875,0.875}}\uline{54.96} & {\cellcolor[rgb]{0.875,0.875,0.875}}\textbf{76.02} & {\cellcolor[rgb]{0.875,0.875,0.875}}\textbf{95.05} & {\cellcolor[rgb]{0.875,0.875,0.875}}\textbf{83.04} & {\cellcolor[rgb]{0.875,0.875,0.875}}\uline{75.70} & {\cellcolor[rgb]{0.875,0.875,0.875}}45.20 & {\cellcolor[rgb]{0.875,0.875,0.875}}\uline{91.93} & {\cellcolor[rgb]{0.875,0.875,0.875}}\textbf{76.33} \\ 
\hhline{~============}
 & \multirow{4}{*}{10-shot} & LoRA & 84.00 & \textbf{77.05} & 54.72 & 74.02 & 94.64 & 82.66 & 76.31 & 46.11 & 91.59 & 75.68 \\
 &  & Adaptformer & 87.08 & 76.96 & 55.69 & 74.78 & 93.69 & 82.45 & 77.07 & 46.29 & 91.38 & 76.15 \\
 &  & Affine-LN & \textbf{87.68} & \uline{76.99} & \textbf{55.91} & \textbf{75.61} & \uline{95.04} & \textbf{83.57} & \textbf{79.27} & \textbf{51.88} & \uline{91.69} & \textbf{77.52} \\
 &  & {\cellcolor[rgb]{0.875,0.875,0.875}}SEA-PEFT & {\cellcolor[rgb]{0.875,0.875,0.875}}\uline{87.22} & {\cellcolor[rgb]{0.875,0.875,0.875}}76.68 & {\cellcolor[rgb]{0.875,0.875,0.875}}\uline{55.71} & {\cellcolor[rgb]{0.875,0.875,0.875}}\uline{75.55} & {\cellcolor[rgb]{0.875,0.875,0.875}}\textbf{95.17} & {\cellcolor[rgb]{0.875,0.875,0.875}}\uline{83.27} & {\cellcolor[rgb]{0.875,0.875,0.875}}\uline{79.23} & {\cellcolor[rgb]{0.875,0.875,0.875}}\uline{48.82} & {\cellcolor[rgb]{0.875,0.875,0.875}}\textbf{91.69} & {\cellcolor[rgb]{0.875,0.875,0.875}}\uline{77.03} \\
\bottomrule
\end{tabular}
}
\vspace{-20pt}
\end{table*}

\begin{table*}[t]
\begin{minipage}[t]{0.55\textwidth}
\centering
\setlength{\extrarowheight}{0pt}
\addtolength{\extrarowheight}{\aboverulesep}
\addtolength{\extrarowheight}{\belowrulesep}
\setlength{\aboverulesep}{0pt}
\setlength{\belowrulesep}{0pt}
\caption{Few-shot efficient adaptation on Totalsegmentator with pre-trained SuPreM \cite{li2024well}.\label{tab:totalseg2}}
\resizebox{1.0\textwidth}{!}{
\small
\renewcommand{\arraystretch}{0.7}
\begin{tabular}{cclcccccc} 
\toprule
\multicolumn{2}{c}{Setting} & \multicolumn{1}{c}{Method} & Gall & Eso & Pan & Duo & Aor & \textbf{Avg.} \\ 
\cline{3-9}
\multirow{10}{*}{\rotatebox[origin=c]{90}{TotalSeg (binary)}} & \multirow{5}{*}{5-shot} & LoRA \cite{lora}& 52.50 & 46.43 & 66.86 & 54.15 & 73.33 & 58.65 \\
 &  & Adaptformer \cite{adaptformer}& 60.17 & 51.79 & 76.73 & 74.49 & \textbf{93.12} & 71.26 \\
 &  & Affine-LN \cite{basu2024strong} & 76.10 & 50.04 & 75.46 & 71.91 & \uline{90.91} & 72.88 \\
 &  & Spatial Adapter \cite{silva2023towards} & \uline{85.08} & \uline{55.56} & \uline{78.84} & \textbf{78.17} & 87.40 & \uline{77.01} \\
 &  & {\cellcolor[rgb]{0.875,0.875,0.875}}SEA-PEFT & {\cellcolor[rgb]{0.875,0.875,0.875}}\textbf{86.67} & {\cellcolor[rgb]{0.875,0.875,0.875}}\textbf{57.32} & {\cellcolor[rgb]{0.875,0.875,0.875}}\textbf{78.87} & {\cellcolor[rgb]{0.875,0.875,0.875}}\uline{75.86} & {\cellcolor[rgb]{0.875,0.875,0.875}}90.74 & {\cellcolor[rgb]{0.875,0.875,0.875}}\textbf{77.89} \\ 
\hhline{~========}
 & \multirow{5}{*}{10-shot} & LoRA \cite{lora}& 59.25 & 55.33 & \uline{77.72} & 73.89 & 80.59 & 69.35 \\
 &  & Adaptformer \cite{adaptformer}& 77.78 & 54.10 & 76.05 & 77.58 & \uline{93.25} & 75.75 \\
 &  & Affine-LN \cite{basu2024strong}& 80.84 & 55.80 & 76.98 & 75.66 & 92.50 & \uline{76.35} \\
 &  & Spatial Adapter \cite{silva2023towards} & \uline{81.61} & \uline{56.24} & 77.69 & \textbf{79.54} & 84.66 & 75.95 \\
 &  & {\cellcolor[rgb]{0.875,0.875,0.875}}SEA-PEFT & {\cellcolor[rgb]{0.875,0.875,0.875}}\textbf{88.51} & {\cellcolor[rgb]{0.875,0.875,0.875}}\textbf{57.88} & {\cellcolor[rgb]{0.875,0.875,0.875}}\textbf{80.98} & {\cellcolor[rgb]{0.875,0.875,0.875}}\uline{78.95} & {\cellcolor[rgb]{0.875,0.875,0.875}}\textbf{93.35} & {\cellcolor[rgb]{0.875,0.875,0.875}}\textbf{79.93} \\
\bottomrule
\end{tabular}
}
\end{minipage}
\hfill
\begin{minipage}[t]{0.4\textwidth}
\centering
\caption{Average times of end-to-end seach-audit-allocate and re-fine-tune with optimal configurations for all experiments.\label{tab:time}}
\resizebox{\textwidth}{!}{
\small
\begin{tabular}{cccc} 
\toprule
Setting & Total steps & Avg. Time & Avg. \#Params (\%) \\ 
\hline
1-shot & 6000 & 2.5 hours &  0.2 \\
5-shot & 8000 & 4.5 hours &  0.2 \\
10-shot & 12000 & 6.5 hours &  0.2 \\
\bottomrule
\end{tabular}
}

\end{minipage}
\end{table*}
Results on pre-trained SuPreM \cite{li2024well} in Table \ref{tab:totalseg2} reinforce the same pattern observed on TotalSegmentator and FLARE'22. SEA-PEFT achieves the highest average Dice in both the 5-shot (77.89) and 10-shot (79.93) settings, with consistent gains on gallbladder, esophagus, and pancreas over Spatial Adapter \cite{silva2023towards}. Notably, the margin over the best fixed baseline widens from 0.88 points at 5-shot to 3.58 points at 10-shot, mirroring the scaling trend across all benchmarks: as more training examples are available, audit signals become less noisy, the allocator converges to better configurations. The advantage of dynamic over static capacity allocation grows.
\begin{figure*}[t]
\centering
  \begin{minipage}[b]{0.32\columnwidth}
  \centering
    \includegraphics[trim={00 1.8cm 0 0},clip,width=0.85\textwidth]{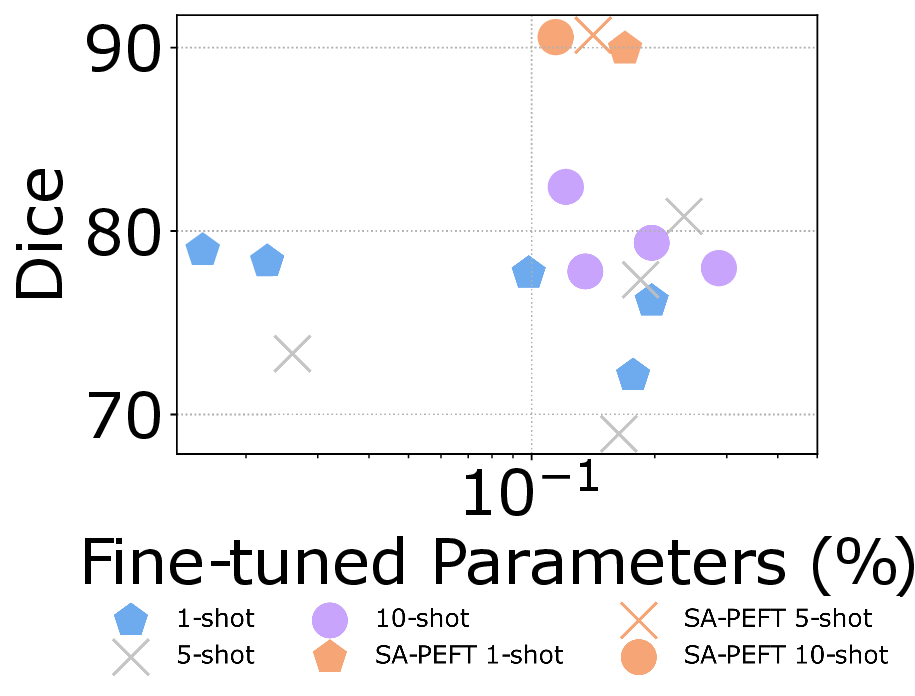}\\ 
    Gal
  \end{minipage}
  \hfill
  \begin{minipage}[b]{0.32\columnwidth}
  \centering
    \includegraphics[trim={0 1.8cm 0 0},clip,width=0.85\textwidth]{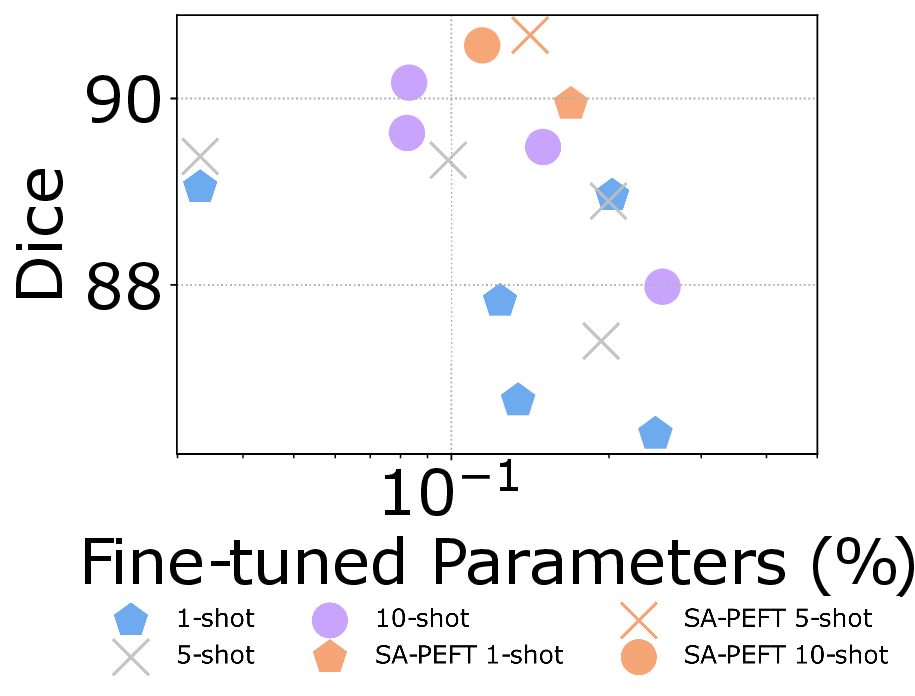}\\
    Kid
    \end{minipage}
  \hfill
  \begin{minipage}[b]{0.32\columnwidth}
  \centering
    \includegraphics[trim={00 1.8cm 0 0},clip,width=0.85\textwidth]{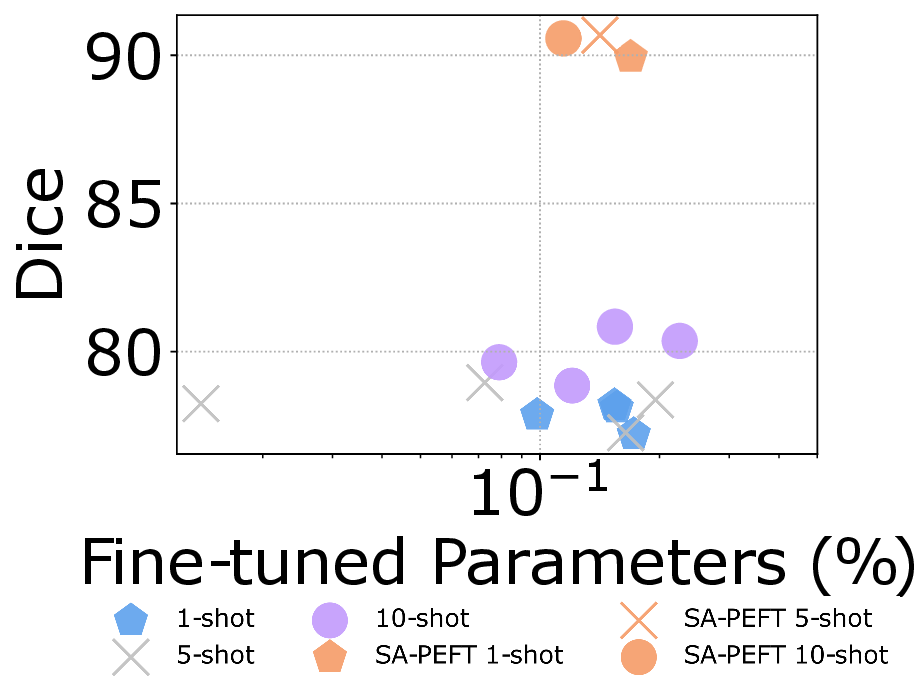}\\
    Pan
\end{minipage} \\
  \begin{minipage}[c]{\columnwidth}
  \centering
\includegraphics[trim={0 0.3cm 0 10.1cm},clip,width=0.6\textwidth]{imgs/gallbladder_rand_size300_2.eps} 
\end{minipage} \\
\caption{SA-PEFT vs. random PEFT configurations under fixed parameter budgets. Across 1/5/10-shot settings, SEA-PEFT’s gains (\textcolor{orange}{Orange}) come from principled selection rather than simply using more parameters.}\label{fig:rand}
\vspace{-5pt}
\end{figure*}
\begin{figure*}[!h]
\centering
\includegraphics[width=\textwidth]{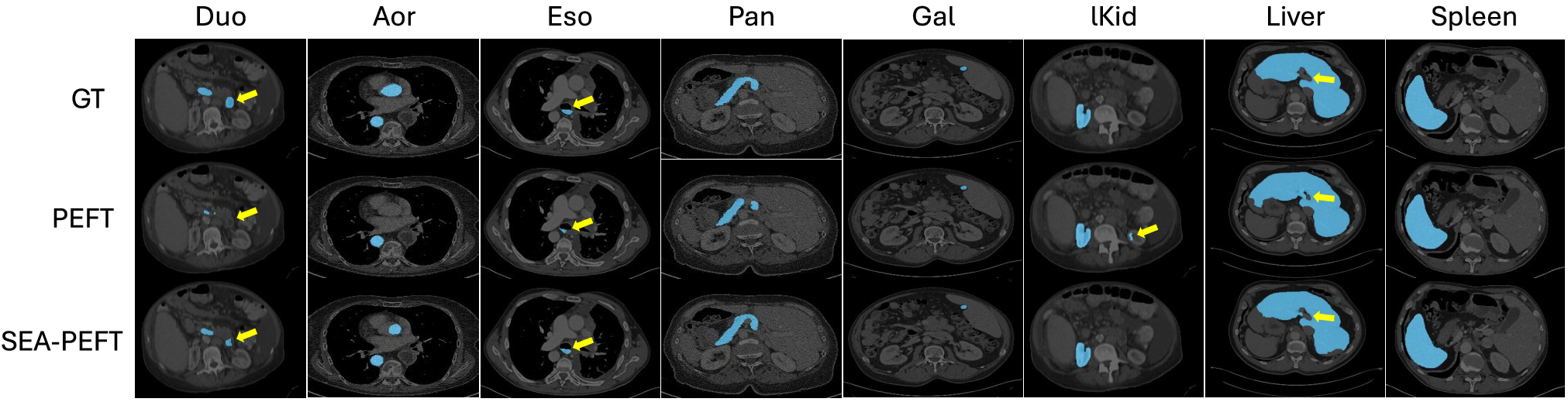} 
\caption{Qualitative segmentation comparison between fixed-topology PEFT and SEA-PEFT across eight organs.}\label{fig:viz}
\vspace{-10pt}
\end{figure*}

\textbf{Runtime and Practicality.} Tab. \ref{tab:time} reports the end-to-end cost of SEA-PEFT, covering the full search–audit–allocate process and a final refinement run. Across 1/5/10-shot settings, adaptation completes within 2.5–6.5 hours while consistently using only 0.2\% trainable parameters. This efficiency stems from SEA-PEFT’s online design: each audit cycle evaluates only a small subset of adapters, so compute scales smoothly with data rather than exploding with configuration space size. In practice, this yields high-quality configurations in hours, making SEA-PEFT realistically deployable in clinical environments where engineers are unavailable and model adaptation must fit within tight timelines.

\textbf{SEA-PEFT vs. Random Configurations.} To examine whether SEA-PEFT’s gains arise from principled selection rather than larger parameter counts, we run a controlled diagnostic on three representative organs, as shown in Fig. \ref{fig:rand}.
SEA-PEFT consistently selects configurations near the upper-left frontier-achieving higher Dice with comparable or fewer parameters. The gap widens as the number of shots increases: with more stable gradients, the audit signal becomes cleaner and SEA-PEFT towards stronger, non-redundant adapter subsets, while random configurations remain broadly scattered. The experiment confirms that SEA-PEFT’s improvements come from targeted, task-level utility estimation rather than sheer parameter volume. 

Fig. \ref{fig:viz} shows qualitative results across all eight organs. SEA-PEFT produces masks that more closely match the ground truth on small, low-contrast structures: duodenum and gallbladder are largely missed by fixed-topology PEFT but partially recovered by SEA-PEFT, while esophagus and left kidney (yellow arrows) show cleaner boundary delineation. On large, high-contrast organs such as liver and spleen, both methods perform comparably, consistent with the quantitative results in Tab. \ref{tab:totalseg}.

\vspace{-10pt}
\section{Conclusion}
We presented SEA-PEFT, a framework that automates PEFT training and configuration search for few-shot 3D medical segmentation by treating adapter selection as an online allocation problem. Each adapter's utility is estimated by on/off Dice perturbation on a validation set, smoothed via EMA and IQR filtering, and used to greedily allocate parameter budget across training. An FSM stabilizer prevents configuration chatter in noisy few-shot settings. On TotalSegmentator and FLARE'22, SEA-PEFT consistently outperforms fixed-topology PEFT baselines across 1/5/10-shot regimes using $<$1\% trainable parameters.

\bibliographystyle{splncs04}
\bibliography{ref}

\end{document}